\documentclass{article} 
\usepackage{times,natbib}
\usepackage{amsmath,amssymb,amsthm}
\usepackage[pdftex,colorlinks]{hyperref}
\usepackage[margin=1.3in]{geometry}
\usepackage{tikz}
\usepackage{pifont}
\newcommand{\xmark}{\ding{55}}
\newcommand{\cmark}{\ding{51}}

\newtheorem{theorem}{Theorem}[section]
\newtheorem{lemma}[theorem]{Lemma}

\newtheorem{note}[theorem]{Note}

\newtheorem{counter-example}[theorem]{Counter example}

\newtheorem{open question}[theorem]{Open question}
\newtheorem{corollary}[theorem]{Corollary}

\newtheorem{definition}[theorem]{Definition}

\newcommand{\reals}{{\mathbb R}}

\newcommand{\cd}{{\mathcal D}}

\newcommand{\cg}{{\mathcal G}}
\newcommand{\ch}{{\mathcal H}}

\newcommand{\cw}{{\mathcal W}}
\newcommand{\cl}{{\mathcal L}}
\newcommand{\cf}{{\mathcal F}}

\newcommand{\cx}{{\mathcal X}}
\newcommand{\cy}{{\mathcal Y}}

\newcommand{\trees}{\mathrm{trees}}

\newcommand{\lab}{\lambda}
\newcommand{\leaf}{\textrm{leaf}}

\newcommand{\sign}{\textrm{sign}}

\newcommand{\dotprod}[1]{\langle #1 \rangle}
\newcommand{\hamming}{\mathrm{\Delta_h}}

\newcommand{\inner}[1]{\langle #1 \rangle}

\DeclareMathOperator*{\Err}{Err}
\DeclareMathOperator*{\VC}{VC}
\DeclareMathOperator*{\argmin}{argmin}
\DeclareMathOperator*{\argmax}{argmax}
\DeclareMathOperator*{\prob}{\mathbb{P}}

\newcommand{\OvA}{\textrm{OvA}}
\newcommand{\AP}{\textrm{AP}}

\newcommand{\secref}[1]{Section~\ref{#1}}
\newcommand{\thmref}[1]{Theorem~\ref{#1}}
\newcommand{\lemref}[1]{Lemma~\ref{#1}}

\title{Multiclass Learning Approaches:\\ A Theoretical Comparison with Implications}

\author{Amit Daniely\thanks{Dept. of Mathematics, The Hebrew University, Jerusalem, Israel} \hspace{1cm} Sivan Sabato\thanks{School of Computer Science and Engineering, The Hebrew University, Jerusalem, Israel} \hspace{1cm} Shai Shalev-Shwartz\thanks{School of Computer Science and Engineering, The Hebrew University, Jerusalem, Israel}
}

\begin{document}

\maketitle

\begin{abstract}
  We theoretically analyze and compare the following five popular multiclass
  classification methods: One vs. All, All Pairs, Tree-based classifiers, Error
  Correcting Output Codes (ECOC) with randomly generated code matrices, and
  Multiclass SVM. In the first four methods, the classification is based on a
  reduction to binary classification. We consider the case where the binary
  classifier comes from a class of VC dimension $d$, and in particular from the
  class of halfspaces over $\reals^d$. We analyze both the estimation error and
  the approximation error of these methods. Our analysis reveals interesting
  conclusions of practical relevance, regarding the success of the different
  approaches under various conditions. Our proof technique employs tools from
  VC theory to analyze the \emph{approximation error} of hypothesis
  classes. This is in sharp contrast to most, if not all, previous uses of VC
  theory, which only deal with estimation error.
%
\end{abstract}

\section{Introduction}

In this work we consider multiclass prediction: The problem of classifying objects into one of several possible target classes. Applications include, for example, categorizing documents according
to topic, and determining which object appears in a given image.  We assume that objects (a.k.a. instances) are vectors in
$\cx = \reals^d$ and the class labels come from the set $\cy = [k] =
\{1,\ldots,k\}$. Following the standard PAC model, the learner
receives a training set of $m$ examples, drawn i.i.d. from some
unknown distribution, and should output a classifier which maps $\cx$ to $\cy$.

The centrality of the multiclass learning problem has spurred the
development of various approaches for tackling the task. Perhaps the
most straightforward approach is a reduction from multiclass classification to
binary classification. For example, the One-vs-All (OvA) method is based on a
reduction of the multiclass problem into $k$ binary problems, each of
which discriminates between one class to all the rest of the classes
(e.g. \cite{RumelhartHiWi86}). A different reduction is the All-Pairs
(AP) approach in which all pairs of classes are compared to each other
\citep{HastieTi98}. These two approaches have been unified under the
framework of Error Correction Output Codes (ECOC)
\citep{DietterichBa95,AllweinScSi00a}. A tree-based classifier (TC) is
another reduction in which the prediction is obtained by traversing
a binary tree, where at each node of the tree a binary classifier is
used to decide on the rest of the path (see for example \cite{BeygelzimerLaRa07}).

All of the above methods are based on reductions to binary
classification. We pay special attention to the case where the
underlying binary classifiers are linear separators
(halfspaces). Formally, each $w\in\mathbb \reals^{d+1}$ defines the
linear separator $h_{w}(x)=\sign(\inner{w,\bar{x}})$, where $\bar{x} =
(x, 1) \in \reals^{d+1}$ is the concatenation of the vector $x$ and the
scalar $1$. While halfspaces are our primary focus, many of our
results hold for any underlying binary hypothesis class of VC
dimension $d+1$.

Other, more direct approaches to multiclass classification over $\reals^d$
have also been proposed (e.g. \cite{Vapnik98,WestonWa99,CrammerSi01a}). In
this paper we analyze the Multiclass SVM (MSVM) formulation of
\cite{CrammerSi01a}, in which each hypothesis is of the form $h_W(x) =
\argmax_{i \in [k]} (W\bar{x})_i$, where $W$ is a $k \times (d+1)$
matrix and $(W\bar{x})_i$ is the $i$'th element of the vector
$W\bar{x} \in \reals^k$.

We theoretically analyze the prediction performance of the
aforementioned methods, namely, OvA, AP, ECOC, TC, and MSVM.
The error of a multiclass predictor $h : \reals^d \to [k]$ is defined
to be the probability that $h(x)\ne y$, where $(x,y)$ is sampled from the underlying
distribution $\cd$ over $\reals^d \times [k]$, namely, $\Err(h) = \prob_{(x,y) \sim \cd}[h(x) \neq y]$.
Our main goal is to understand which method is preferable in terms of the error it will achieve, based on
easy-to-verify properties of the problem at hand.

Our analysis pertains to the type of classifiers each method can
potentially find, and does not depend on the specific training
algorithm. More precisely, each method corresponds to a hypothesis
class, $\ch$, which contains the multiclass predictors that may be
returned by the method. For example, the hypothesis class of MSVM is
$\ch = \{ x \mapsto \argmax_{i \in [k]} (W\bar{x})_i : W \in \reals^{k
  \times (d+1)}\}$.

A learning algorithm, $A$, receives a training set, $S =
\{(x_i,y_i)\}_{i=1}^m$, sampled i.i.d. according to $\cd$, and returns
a multiclass predictor which we denote by $A(S) \in \ch$.  A learning
algorithm is called an Empirical Risk Minimizer (ERM) if it returns a
hypothesis in $\ch$ that minimizes the empirical error on the sample.
We denote by $h^\star$ a hypothesis in $\ch$ with
minimal error,\footnote{For simplicity, we assume that the minimum is
  attainable.} that is, $h^\star \in \argmin_{h \in \ch}
\Err(h)$.

When analyzing the error of $A(S)$, it is convenient to
decompose this error as a sum of \emph{approximation error}
and \emph{estimation error}:
\begin{equation} \label{eqn:2decomposition}
\Err(A(S)) ~=~ \underbrace{\Err(h^\star)}_{\text{approximation}} +
\underbrace{\Err(A(S))-\Err(h^\star)}_{\text{estimation}}.
\end{equation}
\begin{itemize}
\item
The {\bf approximation error} is the minimum error achievable by a
predictor in the hypothesis class, $\ch$.  The approximation error does not
depend on the sample size, and is determined solely by the allowed hypothesis class.
\item
The {\bf estimation error} of an algorithm is the difference between the approximation error,
and the error of the classifier the algorithm chose based on the sample.
This error exists both for statistical reasons, since the sample may not
be large enough to determine the best hypothesis, and for algorithmic reasons,
since the learning algorithm may not output the best possible hypothesis
given the sample. For the ERM algorithm, the estimation error can be bounded from above
by order of $\sqrt{C(\ch)/m}$ where $C(\ch)$ is a complexity measure of $\ch$ (analogous to the VC dimension) and
$m$ is the sample size.
A similar term also bounds the estimation error from below \emph{for any algorithm}.
Thus $C(\ch)$ is an estimate of the best achievable estimation error
for the class.
\end{itemize}

When studying the estimation error of different methods, we follow the
standard distribution-free analysis. Namely, we will compare the
algorithms based on the worst-case estimation error, where worst-case
is over all possible distributions $\cd$. Such an analysis can lead us
to the following type of conclusion: If two hypothesis classes have
roughly the same complexity, $C(\ch_1) \approx C(\ch_2)$, and the
number of available training examples is significantly larger than
this value of complexity, then for both hypothesis classes we are
going to have a small estimation error. Hence, in this case the
difference in prediction performance between the two methods will be
dominated by the approximation error and by the success of the
learning algorithm in approaching the best possible estimation
error. In our discussion below we disregard possible differences in
optimality which stem from algorithmic aspects and implementation
details. A rigorous comparison of training heuristics would certainly
be of interest and is left to future work.

For the approximation error we will provide even stronger results,
by comparing the approximation error of classes for \emph{any}
distribution. We rely on the following definition.
\begin{definition}\label{def:contains}
Given two hypothesis classes, $\ch,\ch'$, we say that $\ch$
  \emph{essentially contains} $\ch'$ if for any distribution, the
  approximation error of $\ch$ is at most the approximation
  error of $\ch'$. $\ch$ {\em  strictly contains} $\ch'$ if, in addition, there is a distribution for which the approximation error of $\ch$ is strictly smaller than that of $\ch'$.
\end{definition}

Our main findings are as follows (see a full comparison in Table~\ref{tab:summary}). The formal statements are given in \secref{sec:main}.
\begin{itemize}
%
\item The estimation errors of OvA, MSVM, and TC are all roughly the
  same, in the sense that $C(\ch) = \tilde{\Theta}(dk)$ for all of the
  corresponding hypothesis classes. The complexity of AP is
  $\tilde{\Theta}(dk^2)$. The complexity of ECOC with a code of length
  $l$ and code-distance $\delta$ is at most $\tilde{O}(dl)$
  and at least $d\delta/2$. It follows that for randomly
  generated codes, $C(\ch) = \tilde{\Theta}(dl)$. Note that this
  analysis shows that a larger code-distance yields a larger
  estimation error and might therefore hurt performance. This contrasts with previous ``reduction-based'' analyses of ECOC, which concluded that a larger code
  distance improves performance.
\item We prove that the hypothesis class of MSVM essentially contains
  the hypothesis classes of both OvA and TC. Moreover, these inclusions are strict.
  Since the estimation
  errors of these three methods are roughly the same, it follows that
  the MSVM method dominates both OvA and TC in terms of achievable prediction performance.
\item In the TC method, one needs to associate each leaf of the tree to a
  label. If no prior knowledge on how to break the symmetry is known, it is
  suggested in \cite{BeygelzimerLaRa07} to break symmetry by choosing a random
  permutation of the labels. We show that whenever $d \ll k$, for any
  distribution $\cd$, with high probability over the choice of a random
  permutation, the approximation error of the resulting tree would be close to
  $1/2$.  It follows that a random choice of a permutation is likely to yield a
  poor predictor.
\item We show that if $d \ll k$, for any distribution $\cd$, the
  approximation error of ECOC with a randomly generated code matrix is  likely to be close to $1/2$.
\item We show that the hypothesis class of AP essentially contains the hypothesis
  class of MSVM (hence also that of OvA and TC), and that there can be a
  substantial gap in the containment. Therefore, as expected, the relative performance of
  AP and MSVM depends on the well-known trade-off between estimation
  error and approximation error.
\end{itemize}
\begin{table}[ht]
\begin{tabular}{|l|c|c|c|c|c|}
\hline
& \textbf{TC} & \textbf{OvA} & \textbf{MSVM} & \textbf{AP} & \textbf{random ECOC}\\
\hline
\textbf{Estimation error} & $dk$ &  $dk$ & $dk$ & $dk^2$ & $dl$\\
\hline
\textbf{Approximation} & $\geq$ MSVM & $\geq$ MSVM & $\geq$ AP & smallest & incomparable \\
\textbf{error} & $\approx 1/2$ when $d \ll k$ & & & & $\approx 1/2$ when $d \ll k$ \\
\hline
\textbf{Testing run-time} & $d\log(k)$ & $dk$ & $dk$ & $dk^2$ & $dl$\\
\hline
\end{tabular}
\label{tab:summary}
\caption{Summary of comparison}
\end{table}

The above findings suggest that in terms of performance, it may be
wiser to choose MSVM over OvA and TC, and especially so when $d \ll k$. We
note, however, that in some situations (e.g. $d =k$) the prediction success of these methods can be similar, while TC has the advantage of having
a testing run-time of $d\log(k)$, compared to the testing run-time of
$dk$ for OvA and MSVM. In addition, TC and ECOC may be a good choice
when there is additional prior knowledge on the distribution or on how
to break symmetry between the different labels.


\subsection{Related work}

\cite{AllweinScSi00a} analyzed the multiclass error of ECOC as a
function of the binary error. The problem with such a
``reduction-based'' analysis is that such analysis becomes problematic
if the underlying binary problems are very hard. Indeed, our analysis
reveals that the underlying binary problems would be too hard if $d
\ll k$ and the code is randomly generated.  The experiments in
\cite{AllweinScSi00a} show that when using kernel-based SVM or
AdaBoost as the underlying classifier, OvA is inferior to random
ECOC. However, in their experiments, the number of classes is small
relative to the dimension of the feature space, especially if working
with kernels or with combinations of weak learners.

\cite{CrammerSi01a} presented experiments demonstrating that MSVM
outperforms OvA on several data sets. \cite{RifkinKl04} criticized the
experiments of \cite{CrammerSi01a,AllweinScSi00a}, and presented
another set of experiments demonstrating that all methods perform
roughly the same when the underlying binary classifier is very strong
(SVM with a Guassian kernel). As our analysis shows, it is not
surprising that with enough data and powerful binary classifiers, all
methods should perform well. However, in many practical applications,
we will prefer not to employ kernels (either because of shortage of
examples, which might lead to a large estimation error, or due to
computational constraint), and in such cases we expect to see a large
difference between the methods.

\cite{BeygelzimerLaRa07} analyzed the \emph{regret} of a specific
training method for trees, called Filter Tree, as a function of the regret of the binary
classifier. The regret is defined to be the difference between the
learned classifier and the Bayes-optimal classifier for the
problem. Here again we show that the regret values of the
underlying binary classifiers are likely to be very large whenever $d
\ll k$ and the leaves of the tree are associated to labels in a random
way. Thus in this case the regret analysis is problematic.  Several
authors presented ways to learn better splits, which corresponds to
learning the association of leaves to labels (see for example
\cite{BengioWeGr11} and the references therein). Some of our negative results
do not hold for such methods, as these do not randomly attach labels to
tree leaves.

\cite{DanielySaBeSh11} analyzed the properties of multiclass learning
with various ERM learners, and have also provided some bounds on the
estimation error of multiclass SVM and of trees. In this paper we both
improve these bounds, derive new bounds for other classes, and also
analyze the approximation error of the classes. To the best of our
knowledge, this is the first case of using VC theory to analyze the
approximation error of hypothesis classes.

\section{Definitions and Preliminaries} \label{sec:preliminaries}

We first formally define the hypothesis classes that we analyze in this paper.
\paragraph{Multiclass SVM (MSVM):} For $W\in \reals^{k\times (d+1)}$
define $h_{W}:\reals^d\to [k]$ by $h_{W}(x)=\text{argmax}_{i\in
  [k]}(W\bar x)_i$ and let $\cl=\{h_{W}:W\in \reals^{k\times
  (d+1)}\}$. Though NP-hard in general, solving the ERM problem with respect to $\cl$ can be
done efficiently in the realizable case (namely, whenever exists a
hypothesis with zero empirical error on the sample).

\paragraph{Tree-based classifiers (TC):} A tree-based multiclass
classifier is a full binary tree whose leaves are associated with
class labels and whose internal nodes are associated with binary
classifiers. To classify an instance, we start with the root node and
apply the binary classifier associated with it. If the prediction is
$1$ we traverse to the right child. Otherwise, we traverse to the left
child. This process continues until we reach a leaf, and then we output
the label associated with the leaf. Formally, a tree for $k$ classes
is a full binary tree $T$ together with a bijection $\lab:\leaf(T)
\rightarrow [k]$, which associates a label to each of the leaves. We
usually identify $T$ with the pair $(T,\lambda)$. The set of
internal nodes of $T$ is denoted by $N(T)$. Let $\ch\subset \{\pm
1\}^\cx$ be a binary hypothesis class. Given a mapping $C:N(T)
\rightarrow \ch$, define a multiclass predictor, $h_{C}:\cx\to[k]$, by
setting $h_C(x)=\lambda(v)$ where $v$ is the last node of the root-to-leaf path $v_1,\ldots v_m=v$ such that $v_{i+1}$ is the left
(resp. right) child of $v_{i}$ if $C(v_{i})(x)=-1$
(resp. $C(v_{i})(x)=1$). Let $\ch_T = \{h_C \mid C:N(T)\to
\ch\}$. Also, let $\ch_{\trees}=\cup_{T\text{ is a tree for $k$ classes
    }}\ch_T$. If $\ch$ is the class of linear separators over $\reals^d$, then for any tree $T$ the ERM problem with respect to $\ch_T$ can be solved efficiently in the realizable case. However, the ERM problem is NP-hard in the non-realizable case.

\paragraph{Error Correcting Output Codes (ECOC):} An ECOC is a code
$M\in \reals^{k\times l}$ along with a bijection $\lambda: [k]\to [k]$. We
sometimes identify $\lambda$ with the identity function and $M$ with
$(M,\lambda)$\footnote{The use of $\lambda$ here allows us to later consider codes with random association of rows to labels.}. Given a code $M$, and the result of $l$ binary classifiers represented by a vector $u \in \{-1,1\}^l$, the code selects a label via $\tilde M:\{-1,1\}^l\to[k]$, defined by $\tilde M(u) =
\lambda\left(\arg\max_{i\in[k]} \sum_{j=1}^l M_{ij}u_j\right)$.  Given binary classifiers $h_1,\ldots,h_l$ for each column in the code matrix,
the code assigns to the instance $x \in \cx$ the label $\tilde M(h_1(x),\ldots,h_l(x))$. Let $\mathcal
H\subset \{\pm 1\}^\cx$ be a binary hypothesis class. Denote by \mbox{$\ch_M\subseteq [k]^\cx$} the hypotheses class
$\ch_M=\{h:\cx \rightarrow [k] \mid \exists (h_1,\ldots,h_l)\in\mathcal H^l \text{ s.t. } \forall x\in\cx, h(x) = \tilde M (h_1(x),\ldots,h_l(x))\}.$

The \emph{distance} of a binary code, denoted by $\delta(M)$ for $M\in \{\pm 1\}^{k\times l}$, is the minimal \emph{hamming distance} between any two pairs of rows in the code matrix. Formally, the hamming distance between $u,v \in \{-1,+1\}^l$ is $\hamming(u,v) = |\{r:u[r]\ne v[r]\}|$, and $\delta(M)=\min_{1\le i<j\le k}\hamming(M[i],M[j])$. The ECOC paradigm described in \citep{DietterichBa95} proposes to choose a code with a large distance.

\paragraph{One vs. All (OvA) and All Pairs (AP):} Let $\ch \subset \{\pm 1\}^\cx$ and $k\ge 2$. In the OvA method we train $k$ binary
problems, each of which discriminates between one class and the
rest of the classes. In the AP approach all pairs of classes are
compared to each other. This is formally defined as two ECOCs. Define $M^{\OvA}\in \reals^{k\times k}$ to be the matrix
whose $(i,j)$ elements is $1$ if $i=j$ and $-1$ if $i \neq j$. Then,
the hypothesis class of OvA is $\ch_{\OvA}=\ch_{M^{\OvA}}$. For
the AP method, let $M^{\AP}\in \reals^{k\times \binom k2}$ be such
that for all $i \in [k]$ and $1 \le j < l \le k$, the coordinate corresponding to row $i$ and column $(j,l)$
is defined to be $-1$ if $i=j$, $1$ if $i=l$, and $0$ otherwise. Then,
the hypothesis class of AP is $\ch_{\AP}=\ch_{M^{\AP}}$.

Our analysis of the estimation error is based on
results that bound the sample complexity of multiclass learning.
The \emph{sample complexity} of an algorithm $A$ is the
function $m_A$ defined as follows: For $\epsilon,\delta>0$,
$m_A(\epsilon,\delta)$ is the smallest integer such that for every
$m\ge m_A(\epsilon,\delta)$ and every distribution $\cd$ on
$\cx\times \cy$, with probability of $>1-\delta$ over the
choice of an i.i.d. sample $S$ of size $m$,
\begin{equation}\label{eq:samplecomplexity}
\Err(A(S_m)) \le \min_{h \in \ch} \Err(h) + \epsilon ~.
\end{equation}
The first term on the right-hand side is the approximation error of
$\ch$. Therefore, the sample complexity is the number of examples
required to ensure that the estimation error of $A$ is at most $\epsilon$ (with high probability).
We denote the sample complexity of a class $\ch$ by
$m_\ch(\epsilon,\delta)=\inf_A m_A(\epsilon,\delta)$, where the
infimum is taken over all learning algorithms.



To bound the sample complexity of a hypothesis class we rely on upper and lower bounds on the
sample complexity in terms of two generalizations of the VC dimension
for multiclass problems, called the \emph{Graph dimension} and the
\emph{Natarajan dimension} and denoted $d_G(\ch)$ and $d_N(\ch)$. For
completeness, these dimensions are formally defined in the appendix.

\begin{theorem}\label{thm:sample-comp-bounds}{\cite{DanielySaBeSh11}}
  For every hypothesis class $\ch$, and for every ERM rule,
$$\Omega\left(\frac{d_N(\ch)+\ln(\frac{1}{\delta})}{\epsilon^2}\right)\le
m_{\ch}(\epsilon,\delta)\le m_{\textrm{ERM}}(\epsilon,\delta)\le O\left(\frac{\min\{d_N(\ch)\ln(|\cy|),d_G(\ch)\}+\ln(\frac{1}{\delta})}{\epsilon^2}\right)$$
\end{theorem}
We note that the constants in the $O,\Omega$ notations are universal.

\section{Main Results} \label{sec:main}

In \secref{sec:sampleTheorems} we analyze
the sample complexity of the different hypothesis classes. We provide lower bounds on the Natarajan dimensions of the various hypothesis classes, thus concluding, in light of
\thmref{thm:sample-comp-bounds}, a lower bound on the sample complexity of \emph{any} algorithm. We also provide upper bounds on the graph dimensions of these hypothesis classes, yielding, by the same theorem, an upper
bound on the estimation error of ERM. In
\secref{sec:approxTheorems} we analyze the approximation error of
the different hypothesis classes.

\subsection{Sample Complexity} \label{sec:sampleTheorems}

Together with Theorem \ref{thm:sample-comp-bounds}, the following
theorems estimate, up to logarithmic factors, the sample complexity
of the classes under consideration. We note that these theorems
support the rule of thumb that the Natarajan and Graph dimensions are
of the same order of the number of parameters. The first theorem shows that the sample complexity of MSVM depends on $\tilde{\Theta}(dk)$.
\begin{theorem}\label{thm:sample_linear}
$d (k-1)\le d_N(\cl)\le d_G(\cl)\le O(dk\log(dk))$.
\end{theorem}

Next, we analyze the sample complexities of TC and ECOC. These methods rely on
an underlying hypothesis class of binary classifiers. While our main focus is
the case in which the binary hypothesis class is halfspaces over $\reals^d$,
the upper bounds on the sample complexity we derive below holds for any binary
hypothesis class of VC dimension $d+1$.

\begin{theorem}\label{thm:sample_FT} For every  binary hypothesis class of VC dimension $d+1$,
and for any tree $T$, $d_G(\ch_T) \leq d_G(\ch_{\trees}) \le O(dk\log(dk))$.
If the underlying hypothesis class is halfspaces over
$\reals^d$, then also
\[
d (k-1)\le d_N(\ch_{T}) \le d_G(\ch_T) \le d_G(\ch_{\trees}) \le
O(dk\log(dk)).
\]
\end{theorem}
Theorems \ref{thm:sample_linear} and \ref{thm:sample_FT} improve
results from \cite{DanielySaBeSh11} where it was shown that $\lfloor
\frac{d}{2}\rfloor \lfloor \frac{k}{2}\rfloor \le d_N(\cl)\le
O(dk\log(dk))$, and for every tree $d_G(\ch_{T})\le O(dk\log(dk))$. Further it was shown that if $\ch$ is the set of halfspaces over $\reals^d$, then $\Omega\left(\frac{d k}{\log(k)}\right)\le
d_N(\ch_{T})$.

We next turn to results for ECOC, and its special cases OvA and AP.
\begin{theorem}\label{thm:sample_ECOC} For every $M\in \reals^{k\times
    l}$ and every binary hypothesis class of VC dimension $d$,
  $d_G(\ch_{M})\le O(dl\log(dl))$. Moreover, if $M\in \{\pm
  1\}^{k\times l}$ and the underlying hypothesis class is
  halfspaces over $\reals^d$, then
\[
d \cdot \delta(M)/2 ~\le~   d_N(\ch_{M}) ~\le~ d_G(\ch_M) ~\le~
O(dl\log(dl)) ~.
\]
\end{theorem}
We note if the code has a large distance, which is the case, for instance, in random codes, then $\delta(M) = \Omega(l)$. In this case, the bound is tight up to
logarithmic factors.

\begin{theorem}\label{thm:sample_OvA}
  For any binary hypothesis class of VC dimension $d$,
  $d_G(\ch_{\OvA})\le O(dk\log(dk))$ and $d_G(\ch_{\AP})\le
  O(dk^2\log(dk))$. If the underlying hypothesis class is halfspaces
  over $\reals^d$ we also have:
  \begin{align*}
&d(k-1)\le d_N(\ch_{\OvA})\le d_G(\ch_{\OvA})\le
O(dk\log(dk)) ~~~~\text{ and }\\
&\textstyle d\binom{k-1}{2}\le d_N(\ch_{\AP})\le d_G(\ch_{\AP})\le O(dk^2\log(dk)).
\end{align*}
\end{theorem}

\subsection{Approximation error} \label{sec:approxTheorems}

We first show that the class $\cl$ essentially contains
$\ch_{\OvA}$ and $\ch_{T}$ for any tree $T$, assuming, of course, that  $\ch$ is the class of
halfspaces in $\reals^d$. We find this result quite surprising, since the sample complexity of all of these classes is of the same order.
\begin{theorem}\label{thm:Linear_contain_FT}
  $\cl$ essentially contains $\ch_{\trees}$ and
  $\ch_{\OvA}$. These inclusions are strict for $d \geq 2$ and $k\ge 3$.
\end{theorem}

One might suggest that a small increase in the dimension would perhaps allow us to embed $\cl$ in $\ch_T$ for some tree $T$ or for OvA. The next result shows that this is not the case.
\begin{theorem}\label{thm:embedding}
Any embedding into a higher dimension that allows $\ch_\OvA$ or $\ch_T$ (for some tree $T$ for $k$ classes)
to essentially contain $\cl$, necessarily embeds into a dimension of at least $\tilde\Omega(dk)$.
\end{theorem}

The next theorem shows that the approximation
error of AP is better than that of MSVM (and hence also better than
OvA and TC). This is expected as the sample complexity of AP is
considerably higher, and therefore we face the usual trade-off between
approximation and estimation error.
\begin{theorem}\label{thm:AP_contain_Linear}
  $\ch_{\AP}$ essentially contains $\cl$. Moreover, there is a constant $k^*>0$, independent of $d$, such that the inclusion is strict  for all $k\ge k^*$.
\end{theorem}
For a random ECOC of length $o(k)$, it is easy to see that it does not contain MSVM, as MSVM has higher complexity. It is also not contained in MSVM, as it generates non-convex regions of labels.

We next derive absolute lower bounds on the approximation errors of
ECOC and TC when $d \ll k$. Recall that both methods are
built upon binary classifiers that should predict $h(x)=1$ if the
label of $x$ is in $L$, for some $L \subset [k]$, and should predict
$h(x)=-1$ if the label of $x$ is not in $L$.
As the following lemma shows, when the
partition of $[k]$ into the two sets $L$ and $[k]\setminus L$ is arbitrary and
balanced, and $k \gg d$, such binary
classifiers will almost always perform very poorly.
\begin{lemma}\label{lemma:split_classes}
  There exists a constant $C>0$ for which the following holds. Let
  $\ch\subseteq \{\pm 1\}^\cx$ be any hypothesis class of VC-dimension
  $d$, let $\mu \in (0,1/2]$, and let $\cd$ be any distribution over
  $\cx\times
  [k]$ such that $\forall i\;\prob_{(x,y)\sim
    \cd}(y=i)\le\frac{10}{k}$. Let $\phi:[k]\to \{\pm 1\}$ be a
  randomly chosen function which is sampled according to one of the
  following rules: (1)
For each $i \in [k]$, each coordinate $\phi(i)$ is chosen
  independently from the other coordinates and
  $\prob(\phi(i)=-1)=\mu$; or
(2) $\phi$ is chosen uniformly among all functions satisfying
  $|\{i\in [k]:\phi(i)=-1\}|=\mu k$.

Let $\cd_\phi$ be the distribution over $\cx\times\{\pm 1\}$ obtained
by drawing $(x,y)$ according to $\cd$ and replacing it with $(x,\phi(y))$.  Then, for any $\nu >0$, if $k\ge
C\cdot\left(\frac{d+\ln(\frac{1}{\delta})}{\nu^2}\right)$, then with
probability of at least $1-\delta$ over the choice of $\phi$, the
approximation error of $\ch$ with respect to $\cd_\phi$ will be at least
$\mu-\nu$.
\end{lemma}

As the corollaries below show, Lemma \ref{lemma:split_classes}
entails that when $k \gg d$, both random ECOCs with a small code length, and balanced trees with a random labeling of the leaves, are expected to perform very poorly.


\begin{corollary}\label{cor:FT_approx_poorly}
  There is a constant $C>0$ for which the following holds. Let
  $(T,\lambda)$ be a tree for $k$ classes such that
  $\lambda:\operatorname{leaf}(T)\to [k]$ is chosen uniformly at
  random. Denote by $k_L$ and $k_R$ the number of leaves of the left
  and right sub-trees (respectively) that descend from root, and let
  $\mu=\min\{\frac{k_1}{k},\frac{k_2}{k}\}$. Let $\ch\subseteq \{\pm
  1\}^\cx$ be a hypothesis class of VC-dimension $d$, let $\nu >0$, and let $\cd$ be any distribution over $\cx\times [k]$
  such that $\forall i\;\prob_{(x,y)\sim
    \cd}(y=i)\le\frac{10}{k}$. Then, for $k\ge C\cdot \left(\frac{d+\ln(\frac{1}{\delta})}{\nu^2}\right)$, with
  probability of at least $1-\delta$ over the choice of $\lambda$, the
  approximation error of $\ch_T$ with respect to $\cd$ is at least
  $\mu-\nu$.
\end{corollary}


\begin{corollary}\label{cor:ECOC_approx_poorly}
  There is a constant $C>0$ for which the following holds. Let
  $(M,\lambda)$ be an ECOC where $M\in \reals^{k\times l}$, and assume that
  the bijection $\lambda: [k]\to [k]$ is chosen uniformly at random. Let $\ch\subseteq \{\pm
  1\}^\cx$ be a hypothesis class of VC-dimension $d$, let $\nu >0$,  and let $\cd$ be any
  distribution over $\cx\times [k]$ such that $\forall
  i\;\prob_{(x,y)\sim \cd}(y=i)\le\frac{10}{k}$. Then, for $k\ge C\cdot
  \left(\frac{dl\log(dl)+\ln(\frac{1}{\delta})}{\nu^2}\right)$, with
  probability of at least $1-\delta$ over the choice of $\lambda$, the
  approximation error of $\ch_M$ with respect to $\cd$ is at least
  $1/2-\nu$.
\end{corollary}
Note that the first corollary holds even if only the top level of the binary
tree is balanced and splits the labels randomly to the left and the right sub-trees. The second corollary holds even if the code itself is not random (nor does it have to be binary), and only the association of rows with labels is random.
In particular, if the length of the code is $O(\log (k))$, as
suggested in \cite{AllweinScSi00a}, and the number of classes is
$\tilde{\Omega}(d)$, then the code is expected to perform poorly.

For an ECOC with a matrix of length $\Omega(k)$ and $d=o(k)$, we do not
have such a negative result as stated in Corollary
\ref{cor:ECOC_approx_poorly}. Nonetheless, \lemref{lemma:split_classes} implies
that the prediction of the binary classifiers when $d = o(k)$ is just slightly better than a
random guess, thus it seems to indicate that the ECOC method will still perform
poorly. Moreover, most current theoretical analyses of ECOC estimate the error
of the learned multiclass hypothesis in terms of the average error of the
binary classifiers. Alas, when the number of classes is large, Lemma
\ref{lemma:split_classes} shows that this average will be close to
$\frac{1}{2}$.

Finally, let us briefly discuss the tightness of Lemma
\ref{lemma:split_classes}. Let $x_1,\ldots,x_{d+1}\in\reals^d$ be affinely
independent and let $\cd$ be the distribution over $\reals^d\times [d+1]$
defined by $\prob_{(x,y)\sim \cd}((x,y)=(x_i,i))=\frac{1}{d+1}$. Is is not hard
to see that for every $\phi:[d+1]\to \{\pm 1\}$, the approximation error of the
class of halfspaces with respect to $\cd_\phi$ is zero. Thus, in order to
ensure a large approximation error \emph{for every distribution}, the number of
classes must be at least linear in the dimension, so in this sense, the lemma
is tight. Yet, this example is very simple, since each
class is concentrated on a single point and the points are linearly
independent. It is possible that in real-world distributions, a
large approximation error will be exhibited even when $k < d$.

We note that the phenomenon of a large approximation error, described in
Corollaries \ref{cor:FT_approx_poorly} and \ref{cor:ECOC_approx_poorly}, does
not reproduce in the classes $\cl,\ch_{OvA}$ and $\ch_{AP}$, since these
classes are symmetric.

\section{Proof Techniques}

Due to lack of space, the proofs for all the results stated above are
provided in the appendix. In this section we give a brief description of our main proof techniques.

Most of our proofs for the estimation error results, stated in Section \ref{sec:sampleTheorems},
are based on a similar method which we now describe. Let $L:\{\pm 1\}^l\to [k]$ be a multiclass-to-binary reduction (e.g., a tree), and for $\ch\subseteq \{\pm 1\}^{\cx}$, denote $L(\ch)=\{x\mapsto L(h_1(x),\ldots,h_l(x))\mid h_1,\ldots,h_l\in\ch\}$.
Our upper bounds for $d_G(L(\ch))$ are mostly based on the following
simple lemma.
\begin{lemma}\label{lemma:meta_lemma}
If $\VC(\ch)=d$ then $d_G(L(\ch))= O(ld\ln(ld))$.
\end{lemma}


The technique for the lower bound on $d_N(L(\cw))$ when $\cw$ is the class of halfspaces in $\reals^d$ is more involved, and quite general. We consider a binary hypothesis class $\cg \subseteq \{\pm1\}^{[d]\times[l]}$ which consists of functions having an arbitrary behaviour over  $[d]\times \{i\}$, and a very uniform behaviour
on other inputs (such as mapping all other inputs to a constant). We show that $L(\cg)$ $N$-shatters the set $[d]\times[l]$. Since $\cg$ is quite simple, this is usually not very hard to show.
Finally, we show that the class of halfspaces is richer than $\cg$, in the sense that the inputs to $\cg$ can be mapped to points in $\reals^d$ such that the functions of $\cg$ can be mapped to halfspaces. We conclude that $d_N(L(\cw))\ge d_N(L(\cg))$.

To prove the approximation error lower bounds stated in Section \ref{sec:approxTheorems}, we use the techniques of VC theory in an unconventional way.
The idea of this proof is as follows:
Using a uniform convergence argument based on the VC dimension of the binary hypothesis class, we show that there exists a small labeled sample $S$ whose approximation error for the hypothesis class is close to the approximation error for the distribution, for all possible label mappings. This allows us to restrict our attention to a finite set of hypotheses, by their restriction to the sample. For these hypotheses,
we show that with high probability over the choice of label mapping, the approximation error on the sample is high. A union bound on the finite set of possible hypotheses shows that the approximation error on the distribution will be high, with high probability over the choice of the label mapping.

\section{Implications} \label{sec:discussion}

The first immediate implication of our results is that whenever the number of examples in the training set is $\tilde{\Omega}(dk)$, MSVM should be preferred to OvA and TC.  This is
certainly true if the hypothesis class of MSVM, $\cl$, has a zero
approximation error (the realizable case), since the ERM is then solvable with respect to $\cl$. Note that since
the inclusions given in \thmref{thm:Linear_contain_FT} are strict, there
are cases where the data is realizable with MSVM but not with
$\ch_{\OvA}$ or with respect to any tree.

In the non-realizable case, implementing the ERM is intractable for
all of these methods. Nonetheless, for each method there are reasonable heuristics to approximate the ERM, which should work well when the approximation
error is small. Therefore, we believe that MSVM should be the
method of choice in this case as well due to its lower approximation error. However, variations in the optimality of algorithms for different hypothesis classes should also be taken into account in this analysis. We leave this detailed analysis of specific training heuristics for future work.
Our analysis also implies that it is highly unrecommended to use TC
with a randomly selected $\lambda$ or ECOC with a random code whenever
$k > d$. Finally, when the number of examples is much larger than $dk^2$, the analysis implies that it is better to choose the AP approach.

To conclude this section, we illustrate the relative performance of MSVM, OvA, TC, and ECOC,
by considering the simplistic case where $d=2$, and each class is
concentrated on a single point in $\reals^2$.  In the leftmost graph
below, there are two classes in $\reals^2$, and the approximation
error of all algorithms is zero. In the middle graph, there are $9$
classes ordered on the unit circle of $\reals^2$. Here, both MSVM and
OvA have a zero approximation error, but the error of TC and of ECOC
with a random code will most likely be large. In the rightmost graph, we chose
random points in $\reals^2$. MSVM still has a zero approximation
error. However, OvA cannot learn the binary problem of distinguishing
between the middle point and the rest of the points and hence has a
larger approximation error.

{
\begin{tabular}{lccc}
\begin{minipage}{0.07\textwidth}
~
\end{minipage}
&
\begin{minipage}{0.27\textwidth}
\begin{center}
\begin{tikzpicture}[scale=0.5]
(1.3,1.3);
\fill (0,0.7) circle (0.1);
\fill (0.7,0) circle (0.1);
\draw[->] (-1.3,0)--(1.3,0);
\draw[->] (0,-1.3)--(0,1.3);
\end{tikzpicture}
\end{center}
\end{minipage}
&
\begin{minipage}{0.27\textwidth}
\begin{center}
\begin{tikzpicture}[scale=0.5]
\foreach \angle in {0,40,...,320}
 \fill (\angle:0.7) circle (0.1);
\draw[->] (-1.3,0)--(1.3,0);
\draw[->] (0,-1.3)--(0,1.3);
\end{tikzpicture}
\end{center}
\end{minipage}
&
\begin{minipage}{0.27\textwidth}
\begin{center}
\begin{tikzpicture}[scale=0.5]
 \fill (20:0.8) circle (0.1);
 \fill (50:1) circle (0.1);
 \fill (80:0.2) circle (0.1);
 \fill (100:1.1) circle (0.1);
 \fill (150:0.7) circle (0.1);
 \fill (200:0.7) circle (0.1);
 \fill (250:1) circle (0.1);
 \fill (300:0.9) circle (0.1);
\draw[->] (-1.3,0)--(1.3,0);
\draw[->] (0,-1.3)--(0,1.3);
\end{tikzpicture}
\end{center}
\end{minipage}
\\[0.3cm] & & &  \\ \hline
MSVM & \cmark & \cmark & \cmark \\
OvA & \cmark & \cmark & \xmark \\
TC/ECOC & \cmark & \xmark & \xmark
\end{tabular}
}

\bibliographystyle{plainnat}
\bibliography{bib}

\appendix

\section{Proofs}
\subsection{Notation and Definitions}
Throughout the proofs, we fix $d,k\ge 2$. We denote by $\cw=\cw^d=\{h_w:w\in \reals^{d+1}\}$ the class of linear separators (with bias) over $\reals^d$.
We assume the following "tie breaking" conventions:
\begin{itemize}
\item For $f:[k]\to \mathbb R$, $\text{argmax}_{i\in[k]}f(i)$ is the {\em minimal} number $i_0\in[k]$ for which $f(i_0)=\max_{i\in[k]}f(i)$;
\item $\sign(0)=1$.
\end{itemize}
Given a hypotheses class $\ch\subseteq \cy^\cx$, denote its restriction to $A\subseteq \cx$ by $\ch|_A=\{f|_A:f\in\ch\}$.
Let $\ch\subseteq \cy^\cx$ be a hypothesis class and let $\phi:\cy\to\cy'$, $\iota:\cx\to\cx'$ be functions. Denote $\phi\circ\ch=\{\phi\circ h:h\in\ch\}$ and $\ch\circ \iota =\{h\circ\iota:h\in\ch\}$.

Given $\ch\subseteq \cy^\cx$ and a distribution $\cd$ over $\cx\times\cy$, denote the approximation error by $\Err^*_\cd(\ch)=\inf_{h\in\ch}\Err_\cd(h)$. Recall that by definition \ref{def:contains}, $\ch$ essentially contains $\ch'\subseteq \cy^\cx$ if and only if $\Err^*_\cd(\ch)\le\Err^*_\cd(\ch')$ for every distribution $\cd$.
For a binary hypothesis class $\ch$, denote its VC dimension by $\VC(\ch)$.

Let $\ch\subseteq \cy^\cx$ be a hypothesis class and let $S\subseteq
\cx$. We say that $\ch$ \emph{G-shatters} $S$ if there exists an
$f:S\rightarrow \cy$ such that for every $T\subseteq S$ there is a
$g\in \ch$ such that
\[
\forall x\in T,\: g(x)=f(x),\text{ and \:}\forall x\in S\setminus T,\: g(x)\ne f(x).
\]
We say that $\ch$ \emph{N-shatters} $S$ if there exist $f_1,f_2: S \rightarrow \cy$ such that $\forall y\in S,\; f_1(y)\ne f_2(y)$, and for every $T\subseteq S$ there is a $g\in \ch$ such that
\[
\forall x\in T,\: g(x)=f_1(x),\text{ and \:}\forall x\in S\setminus T,\: g(x)=f_2(x).
\]
The \emph{graph dimension} of $\ch$, denoted $d_G(\ch)$, is the
maximal cardinality of a set that is G-shattered by $\ch$. The
\emph{Natarajan dimension} of $\ch$, denoted $d_N(\ch)$, is the
maximal cardinality of a set that is N-shattered by $\ch$.  Both of
these dimensions coincide with the VC-dimension for $|\cy|=2$. Note
also that we always have $d_N(\ch)\le d_G(\ch)$. As shown in
\cite{Ben-DavidCeHaLo95}, it also holds that $d_G(\ch)\le
4.67\log_2(|Y|)d_N(\ch)$.

\begin{proof}[Proof of Lemma \ref{lemma:meta_lemma}]
Let $A\subseteq \cx$ be a $G$-shattered set with $|A|=d_G(L(\ch))$. By Sauer's Lemma, $2^{|A|}\le |\ch|_{A}|^l\le |A|^{dl}$, thus $d_G(L(\ch))=|A|=O(ld\log(ld))$.
\end{proof}

\subsection{Multiclass SVM}\label{sec:SVM}
\begin{proof}[Proof of Theorem \ref{thm:sample_linear}]
The lower bound follows from Theorems \ref{thm:Linear_contain_FT} and \ref{thm:sample_FT}. To upper bound $d_G:=d_G(\cl)$, let $S = \{x_1,\ldots,x_{d_G}\} \subseteq \reals^d$ be a
set which is G-shattered by $\cl$, and let $f:S \rightarrow [k]$ be
a function that witnesses the shattering. For $x\in \reals^d$ and $j\in [k]$, denote
\[
\phi(x,j)=(0,\ldots 0,x[1],\ldots,x[d],1,0,\ldots,0)\in \reals^{(d+1)k},
\]
where $x[1]$ is in the $(d+1)(j-1)$ coordinate. For every $(i,j)\in [d_G]\times [k]$, define $z_{i,j} =\phi(x_i,f(x_i)) - \phi(x_i,j)$.
Denote $Z = \{z_{i,j} \mid (i,j)\in [d_G]\times [k]\}$. Since $\VC(\cw^{(d+1)k})=(d+1)k+1$, by Sauer's lemma,
\[
|\cw^{(d+1)k}|_Z|\le |Z|^{(d+1)k+1} = (d_Gk)^{(d+1)k+1}.
\]
We now show that
there is a one-to-one mapping from subsets of $S$ to $\cw^{(d+1)k}|_Z$,
thus concluding an upper bound on the size of $S$.
For any $T \subseteq S$, choose $W(T) \in \reals^{k\times (d+1)}(\reals)$ such that
\begin{align*}
\{x \in S \mid h_{W(T)}(x) = f(x)\} = T.
\end{align*}
Such a $W(T)$ exists because of the G-shattering of $S$ by $\cl$ using the witness $f$.
Define the vector $w(T)\in \reals^{k(d+1)}$ which is the concatenation of the rows of $W(T)$, that is  
$$w(T)=(W(T)_{(1,1)},\ldots,W(T)_{(1,d+1)},\ldots,W(T)_{(k,1)},\ldots,W(T)_{(k,d+1)})$$.
Now, suppose that $T_1\ne T_2$ for $T_1,T_2\subseteq S$. We now show that $w(T_1)|_Z \neq w(T_2)|_Z$.
Suppose w.l.o.g. that there is some $x_i\in T_1\setminus T_2$. Thus, $f(x_i)=h_{W(T_1)}(x_i)\ne h_{W(T_2)}(x_i)=:j$. It follows that the inner product of $x_i$ with row $f(x_i)$ of $W(T_1)$ is greater than the inner product of $x_i$ with row $j$ of $W(T_1)$, while for $W(T_2)$, the situation is reversed. Therefore,
$\sign(\langle w(T_1),z_{i,j}\rangle)\ne \sign(\langle w(T_2),z_{i,j}\rangle)$, so $w(T_1)$ and $w(T_2)$ induce different labelings of $Z$. It follows that the number of subsets of $S$ is bounded by the size of $\cw^{(d+1)k}|_Z$, thus $2^{d_G} \leq (kd_G)^{(d+1)k+1}$. We conclude that $d_G \leq O(dk\log(dk))$.
\end{proof}

\subsection{Simple classes that can be represented by the class of linear separators}\label{sec:reduction_classes}
In this section we define two fairly simple hypothesis classes, and show that the class of linear separators is richer than them. We will later use this observation to prove lower bounds on the Natarajan dimension of various multiclass hypothesis classes.

Let $l\ge 2$. For $f\in \{-1,1\}^{[d]},\;i\in[l],\;j\in \{-1,1\}$ define $f^{i,j}:[d]\times [l]\to\{-1,1\}$ by
\[
f^{i,j}(u,v)=\begin{cases}
f(u) & v=i\\
j & v\ne i,
\end{cases}
\]
And define the hypothesis class $\cf^l$ as
\[
\cf^{l}=\{f^{i,j}: f\in \{\pm 1\}^{[d]},\;i\in[l],\;j\in \{-1,1\}\}.
\]
For $g\in \{-1,1\}^{[d]},\;i\in[l],\;j\in \{\pm 1\}$ define $g^{i,j}:[d]\times [l]\to\{-1,1\}$ by
\[
g^{i,j}(u,v)=\begin{cases}
h(u) & v=i\\
 j & v> i\\
 -j & v< i,\end{cases}
\]
And define the hypothesis class $\cg^l$ as
\[
\cg^{l}=\{g^{i,j}: g\in \{-1,1\}^{[d]},\;i\in[l],\;j\in \{\pm 1\}\}.
\]

Let $\ch\subset\cy^\cx,\ch'\subset\cy^{\cx'}$ be two hypotheses classes. We say that $\ch$ is {\em richer} than $\ch '$ if there is a mapping $\iota:\cx'\to\cx$ such that $\ch'=\ch\circ \iota$.
It is clear that if $\ch$ is richer than $\ch'$ then $d_N(\ch')\le d_N(\ch)$ and $d_G(\ch')\le d_G(\ch)$. Thus, the notion of richness can be used to establish lower and upper bounds on the Natarajan and Graph dimension, respectively. The following lemma shows that $\cw$ is richer than $\cf^{l}$ and $\cg^{l}$ for every $l$. This will allow us to use the classes $\cf^{l},\;\cg^{l}$ instead of $\cw$ when bounding from below the dimension of an ECOC or TC hypothesis class
in which the binary classifiers are from $\cw$.
\begin{lemma}\label{lem:linear-sep-are-reach}
For any integer $l\geq 2$, $\cw$ is richer than $\cf^{l}$ and $\cg^{l}$.
\end{lemma}
\begin{proof} We shall first prove that $\cw$ is richer than $\cf^{l}$. Choose $l$ unit vectors $e_1,\ldots, e_l\in \mathbb R^d$. For every $i\in [l]$, choose $d$ affinely independent vectors such that
\[
x_{1,i},\ldots ,x_{d,i}\in \{x\in \reals^d:\langle x,e_i\rangle=1,\;\forall i'\ne i,\langle x,e_{i'}\rangle<1\}.
\]
This can be done by choosing $d$ affinely independent vectors in $\{x\in \reals^d:\langle x,e_i\rangle=1\}$ that are very close to $e_i$. Define $\iota(m,i)=x_{m,i}$.
Now fix $i \in [l]$ and $j \in \{-1,+1\}$, and let $f^{i,j}\in \cf^{l}$. We must show that $f^{i,j}=h\circ\iota$ for some $h\in\cw$. 
We will show that there exists an affine map $\Lambda:\reals^d\to\reals$ for which $f^{i,j}=\sign\circ \Lambda \circ \iota$. This suffices, since $\cw$ is exactly the set of all functions of the form $\sign\circ \Lambda $ where $\Lambda$ is an affine map.
Define \mbox{$M=\{x\in \reals^d:\langle x,e_i\rangle=1\}$},
and let $A:M\to \reals$ be the affine map defined by
\[
\forall m\in [d],\;A(x_{m,i})=f(m,i).
\] Let $P:\reals^d\to M$ be the orthogonal projection of $\reals^d$ on $M$. For $\alpha\in\reals$, define an affine map $\Lambda_{\alpha}:\reals^d\to\reals$ by
\[
\Lambda_{\alpha}(x)=A(P(x))+\alpha\cdot \langle x-e_i,e_i\rangle.
\]
Note that, $\forall m\in [d],\;\Lambda_{\alpha}(x_{m,i})=f(m,i)$. Moreover, for every $i'\ne i$ and $m\in[d]$ we have $\langle x_{m,i'}-e_i,e_i\rangle<0$. Thus, by choosing $|\alpha|$ sufficiently large and choosing $\sign(\alpha)$ depending on $j$, we can make sure that $f^{i,j}=\sign\circ \Lambda_{\alpha}\circ \iota$.

The proof that $\cw$ is richer than $\cg^{l}$ is similar and simpler. Let $e_1,\ldots ,e_d\in\mathbb R^{d-1}$ be affinely independent. Define
\[
\iota(m,i)=(e_m,i)\in\reals^{d-1}\times\reals \cong \reals^d,
\]
Given $g^{i,j}\in \cg^{d,l}$, let $A:\reals^{d-1}\times\{i\}\to \reals$ be the affine map defined by $A(e_m,i)=g^{i,j}(m,i)$ and let $P:\reals^d\to\reals^{d-1}\times\{i\}$ be the orthogonal projection. Define $\Lambda:\reals^d\to\reals$ by
\[
\Lambda (x,y)=A(P(x,y))+j\cdot 10\cdot (y-i).
\]
It is easy to check that $\sign\circ\Lambda\circ\iota=g^{i,j}$.
\end{proof}

\begin{note}
From Lemma \ref{lem:linear-sep-are-reach} it follows that $\VC(\cf^{l}),\VC(\cg^{l})\le d+1$. On the other hand, both $\cf^{l}$ and $\cg^{l}$ shatter $([d]\times \{1\})\cup \{(1,2)\}$. Thus, $\VC(\cf^{l})=\VC(\cg^{l})=d+1$
\end{note}

\subsection{Trees}\label{sec:FT}
\begin{proof}[Proof of Theorem \ref{thm:sample_FT}] We first prove the upper bound. Let $A\subseteq \cx$ be a $G$-shattered set with $|A|=d_G(\ch_{\trees})$. By Sauer's Lemma, and since the number of trees is bounded by $k^k$, we have $$2^{|A|}\le k^k\cdot|\ch|_{A}|^k\le k^k\cdot|A|^{dk},$$thus $d_G(\ch_{\trees})=|A|=O(dk\log(dk))$.

To prove the lower bound, by Lemma \ref{lem:linear-sep-are-reach}, it is enough to show that $d_N(\cg^l_T)\ge d\cdot (k-1)$ for some $l$. We will take $l=|N(T)|=k-1$. Linearly order $N(T)$ such that for every node $v$, the nodes in the left sub-tree emanating from $v$ are smaller than the nodes in the corresponding right sub-tree. We will identify $[l]$ with $N(T)$ by an order-preserving map, thus $\cg^{l}\subset \{-1,1\}^{[d]\times N(T)}$. We also identify the labels with the leaves.

Define $g_1:[d]\times N(T)\to \operatorname{leaf}(T)$ by setting $g_1(i,v)$ to be the leaf obtained by starting from the node $v$, going right once and then going left until reaching a leaf. Similarly, define $g_2:[d]\times N(T)\to leaf(T)$ by setting $g_2(i,v)$ to be the leaf obtained by starting from the node $v$, going left once and then going right until reaching a leaf.

We shall show that $g_1,g_2$ witness the $N$-shattering of $[d]\times N(T)$ by $\cg^{l}_T$. Given $S\subset [d]\times N(T)$ define $C:N(T)\to\cg^{l}$ by
$$C(v)(i,u)=\begin{cases}
-1 & u<v\\
1 & u>v\\
1& u= v,\;(i,u)\in S\\
-1& u= v,\;(i,u)\notin S.\\
\end{cases}$$
It is not hard to check that $\forall (i,u)\in S,\;h_C(i,u)=g_1(i,u)$,
and $\forall (i,u)\notin S,\;h_C(i,u)=g_2(i,u)$.
\end{proof}
\begin{note}
Define $\tilde{\cg}^{l}=\{g^{i,1}: g\in \{-1,1\}^{[d]},\;i\in[l]\}.$
The proof shows that \mbox{$d_N(\tilde{\cg}^{l}_T)\ge d\cdot (k-1)$}. Since $\VC(\tilde{\cg}^{l})=d$, we obtain a simpler proof of Theorem 23 from \cite{DanielySaBeSh11}, which states that for every tree $T$ there exists a class $\ch$ of VC dimension $d$ for which $d_N(\ch_T)\ge d(k-1)$.
\end{note}

\subsection{ECOC, One vs. All and All Pairs}\label{sec:ECOC}
To prove the results for ECOC and its special cases,
we first prove a more general theorem, based on the notion of a sensitive vector for a given code. Fix a code $M\in \reals^{k\times l}(\reals)$. We say that a binary vector $u\in \{\pm1\}^l$ is {\em q-sensitive} for $M$ if there are $q$ indices $j\in [l]$ for which $\tilde M(u)\ne \tilde M(u\oplus e_j)$. Here, $u\oplus e_j:=(u[1],\ldots,-u[j],\ldots,u[l])$.

\begin{theorem}\label{thm:sensitive_vec}
If there exists a $q$-sensitive vector for a code $M\in \reals^{k\times l}(\reals)$ then $d_N(\cw_M)\ge d \cdot q$.
\end{theorem}

\begin{proof}
By Lemma \ref{lem:linear-sep-are-reach}, it suffices to show that
$d_N(\cf^{l}_M)\ge d \cdot q$.
Let $u\in\{\pm1\}^l$ be a $q$-sensitive vector. Assume w.l.o.g. that the sensitive coordinates are $1,\ldots,q$. We shall show that $[d]\times [q]$ is $N$-shattered by $\cf^{l}_M$. Define $g_1,g_2:[d]\times [q]\to [k]$ by
$$g_1(x,y)=\tilde M(u),\;g_2(x,y)=\tilde M(u\oplus e_y)$$

Let $T\subset [d]\times[q]$. Define $h_1,\ldots,h_l\in \cf^l$ as follows. For every $j>q$, define $h_j\equiv u[j]$. For $j\le q$ define
\[
h_j(x,y)=\begin{cases}
u[j] & y\neq j\\
u[j]& y= j,\;(x,y)\in T\\
-u[j]& y= j,\;(x,y)\in [d]\times[q]\setminus T.\\
\end{cases}
\]
For $h=(h_1,\ldots,h_l)$, it is not hard to check that
\begin{align*}
\forall (x,y)\in T,\quad&\tilde{M}(h_1(x,y),\ldots,h_l(x,y))=g_1(x,y),\text{ and }\\
\forall (x,y)\in [d]\times[q]\setminus T,\quad&\tilde{M}(h_1(x,y),\ldots,h_l(x,y))=g_2(x,y).
\end{align*}
\end{proof}

The following lemma shows that a code with a large distance is also highly sensitive. In fact, we prove a stronger claim: the sensitivity is actually
at least as large as the distance between any row and the row closest to it in Hamming distance. Formally, we consider $\Delta(M)=\max_{i}\min_{j\ne i}\hamming(M[i],M[j]) \ge \delta(M)$.
\begin{lemma}\label{lemma:sensitive_vec}
For any code $M\in \reals^{k\times l}(\pm 1)$, there is a $q$-sensitive vector for $M$, where \mbox{$q\ge \frac{1}{2}\Delta(M) \geq \frac12 \delta(M)$}.
\end{lemma}
\begin{proof}
Let $i_1$ the row in $M$ such that its hamming distance to the row closest to it is $\Delta(M)$. Denote by $i_2$ the index of the closest row (if there is more than one such row, choose one of them arbitrarily). We have $\hamming(M[i_1],M[i_2]) = \Delta(M)$. In addition, \mbox{$\forall i\neq i_1,i_2, \hamming(M[i_1],M[i]) \geq \Delta(M)$}.
Assume w.l.o.g. that the indices in which rows $i_1$ and $i_2$ differ are $1,\ldots,\Delta(M)$. Consider first the case that $i_1<i_2$. Define $u\in\{\pm 1\}^{[l]}$ by
$$u[j]=\begin{cases}
M_{(i_1,j)} & j\le \lceil \frac{\Delta}{2}\rceil\\
M_{(i_2,j)} & \text{otherwise}.
\end{cases}$$
Is is not hard to check that for every $1\le j\le \lceil \frac{\Delta}{2}\rceil$, $i_1=\tilde M(u)$ and $\tilde M(u\oplus e_j)=i_2$, thus $u$ is $\lceil \frac{\Delta}{2}\rceil$-sensitive. If $i_1>i_2$, the proof is similar except that $u$ is defined as $$u[j]=\begin{cases}
M_{(i_2,j)} & j\le \lceil \frac{\Delta}{2}\rceil\\
M_{(i_1,j)} & \text{otherwise}.
\end{cases}$$
\end{proof}

\begin{proof}[Proof of Theorem \ref{thm:sample_ECOC}]
The upper bound follows from Lemma \ref{lemma:meta_lemma}.
The lower bound follows form Theorem \ref{thm:sensitive_vec} and Lemma \ref{lemma:sensitive_vec}.
\end{proof}

\begin{proof}[Proof of Theorem \ref{thm:sample_OvA}] The upper bounds follow from Theorem \ref{thm:sample_ECOC}. To show that $d_N(\cw_{\OvA})\ge (k-1)d$, we note that the all-negative vector $u = (-1,\ldots,-1)$ of length $k$ is $(k-1)$-sensitive for the code $M^{\OvA}$, and apply Theorem \ref{thm:sensitive_vec}.

To show that $d_N(\cw_{\AP})\ge d\binom {k-1}2$, assume for simplicity that $k$ is odd (a similar analysis can be given when $k$ is even). Define $u\in \{\pm 1\}^{\binom k2}$ by
$$\forall i<j,\;u[i,j]=\begin{cases}
1 & j-i\le  \frac{k-1}{2}\\
-1 & \text{otherwise}.
\end{cases}$$

For every $n\in [k]$, we have $\sum_{1\le i<j\le k}u[i,j]\cdot M^{\AP}_{n,(i,j)}=0$, as the summation counts the number of pairs $(i,j)$ such that $n\in \{i,j\}$ and $M^{\AP}_{n,(i,j)}$ agrees with $u[i,j]$. Thus, $\tilde M^{\AP}(u)=1$, by our tie-breaking assumptions. Moreover, it follows that for every \mbox{$1<i<j\le k$}, we have  $\tilde M^{\AP}(u\oplus e_{(i,j)})\in \{i,j\}$, since flipping entry $[i,j]$ of $u$ increases $(M^{AP}u)_j$ or $(M^{AP}u)_i$ by $1$ and does not increase the rest of the coordinates of the vector $M^{AP}u$. This shows that $u$ is $\binom {k-1}2$-sensitive.
\end{proof}

\subsection{Approximation}\label{sec:approx}
\begin{proof}[Proof of Theorem \ref{thm:Linear_contain_FT}]
We first show that for any tree for $k$ classes $T$, $\cl$ essentially contains $\cw_T$. It follows that $\cl$ essentially contains $\cw_{\trees}$ as well.  Let $\cd$ a distribution over $\reals^d$, let $C:N(T)\to \cw$ be a mapping associating nodes in $T$ to binary classifiers in $\cw$, and let $\epsilon>0$. We will show that there exists a matrix $W\in \reals^{k\times (d+1)}$ such that $\Pr_{x\sim \cd}[h_W(x)\ne h_C(x)]<\epsilon$.

For every $v\in N(T)$, denote by $w(v)\in \reals^{d+1}$ the linear separator such that $C(v)=h_{w(v)}$. For every $w\in\reals^{d+1}$ define $\tilde w=w+(0,\ldots,0,\gamma)$. Recall that for $x \in \reals^d$, $\bar{x} \in \reals^{d+1}$ is simply the concatenation $(x,1)$.
Choose $r>0$ large enough so that $\Pr_{x\sim \cd}[||\bar x||> r]<\epsilon/2$ and $\forall v\in N(T),\;||\tilde w(v)||<r$. Choose $\gamma>0$ small enough so that \[
\Pr_{x\sim\cd}\left[ \exists v\in N(T), \langle \tilde w(v),\bar x \rangle \in (-\gamma,\gamma)  \right]=\Pr_{x\sim\cd}\left[ \exists v\in N(T), \langle w(v),\bar x \rangle \in (-2\gamma,0)  \right]<\epsilon/2.
\]
Let $a=2r^2/\gamma+1$.
For $i\in [k]$, let $v_{i,1},\ldots,v_{i,{m_i}}$ be the path from the root to the leaf associated with label $i$. For each $1\le j<m_i$ define $b_{i,j}=1$ if $v_{i,j+1}$ is the right son of $v_{i,j}$, and $b_{i,j}=-1$ otherwise.
 Now, define $W\in \reals^{k\times (d+1)}$ to be the matrix whose $i$'th row is $w_i=\sum_{j=1}^{m_i-1}a^{-j}\cdot b_{i,j}\tilde w(v_{i,j})$.

To prove that $\Pr_{x\sim \cd}[h_W(x)\ne h_C(x)]<\epsilon$, it suffices to show that $h_W(x)=h_C(x)$ for every $x\in\reals^d$ satisfying $||\bar x||<r$ and $\forall v \in N(T), \langle \tilde w(v),\bar x \rangle\notin (-\gamma,\gamma)$, since the probability mass of the rest of the vectors is less than $\epsilon$. Let $x\in\reals^d$ be a vector that satisfies these assumptions. Denote $i_1=h_C(x)$. It suffices to show that for all $i_2\in [k]\setminus \{i_1\},\,\langle w_{i_1}, \bar x\rangle>\langle w_{i_2}, \bar x\rangle$,
since this would imply that $h_W(x) = i_1$ as well.

Indeed, fix $i_2 \neq i_1$, and let $j_0$ be the length of the joint prefix of the two root-to-leaf paths that match the labels $i_1$ and $i_2$.
In other words, $\forall j\le j_0,\;v_{i_1,j}=v_{i_2,j}$ and $v_{i_1,j_0+1}\ne v_{i_2,j_0+1}$. Note that
\[
\langle\bar x, (b_{i_1,j_0}-b_{i_2,j_0})\tilde w(v_{i_1,j_0})\rangle=\langle\bar x, 2b_{i_1,j_0}\tilde w(v_{i_1,j_0})\rangle=2|\langle\bar x, \tilde w(v_{i_1,j_0})\rangle |\ge 2\gamma.
\]
The last equality holds because $b_{i_1,j_0}$ and $\langle \bar x, w(v_{i_1,j_0}) \rangle$ have the same sign by definition of $b_{i,j}$.
We have
\begin{align*}
&\langle w_{i_1}, \bar x\rangle-\langle w_{i_2}, \bar x\rangle = \langle \bar x, \sum_{j=1}^{m_{i_1}-1}a^{-j} b_{i_1,j} \tilde w(v_{i_1,j})-\sum_{j=1}^{m_{i_2}-1}a^{-j} b_{i_2,j}\tilde w(v_{i_2,j})\rangle
\\
&\qquad= \langle \bar x,a^{-{j_0}}(b_{i_1,j_0}-b_{i_2,j_0})\tilde{w}(v_{i_1,j_0})\rangle+ \langle \bar x, \sum_{j=j_0+1}^{m_{i_1}-1}a^{-j} b_{i_1,j} \tilde{w}(v_{i_1,j})-\sum_{j=j_0+1}^{m_{i_2}-1}a^{-j} b_{i_2,j}\tilde w(v_{i_2,j})\rangle
\\
&\qquad\ge \langle \bar x,a^{-{j_0}}(b_{i_1,j_0}-b_{i_2,j_0})\tilde{w}(v_{i_1,j_0})\rangle-\sum_{j=j_0+1}^\infty a^{-j}2r^2
\\
&\qquad\ge  2a^{-j_0}\left(\gamma-\frac{r^2}{a-1}\right)>0.
\end{align*}
Since this holds for all $i_2 \neq i_1$, it follows that $h_W(x) = i_1$. Thus, we have proved that $\cl$ essentially contains $W_{\trees}$.

Next, we show that $\cl$ strictly contains $\cw_\trees$, by showing a distribution over labeled examples such that the approximation error using $\cl$ is strictly smaller than the approximation error using $\cw_\trees$.
Assume w.l.o.g.~that $d=2$ and $k=3$: even if they are larger we can always restrict the support of the distribution to a subspace of dimension $2$ and to only three of the labels.
Consider the distribution $\cd$ over $\reals^2\times [3]$ such that its marginal over $\reals^2$ is uniform
in the unit circle, and $\Pr_{(X,Y) \sim D}[Y = i \mid X = x] = \mathbb{I}[x \in D_i]$,
where $D_1,D_2,D_3$ be subsets sectors of equal angle of the unit circle (see Figure \ref{fig:dist}):

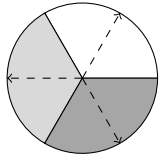
\begin{figure}[ht]
\centering
\begin{tikzpicture}
\draw[fill=white] (0,0)-- (0:1cm)
arc (0:120:1cm) -- cycle;
\draw[fill=gray!30] (0,0)-- (120:1cm)
arc (120:240:1cm) -- cycle;
\draw[fill=gray!70] (0,0)-- (240:1cm)
arc (240:360:1cm) -- cycle;
\draw[dashed,->] (0,0)--(0.5,0.866);
\draw[dashed,->] (0,0)--(0.5,-0.866);
\draw[dashed,->] (0,0)--(-1,0);
\end{tikzpicture}
\caption{Illustration for the proof of Theorem \ref{thm:Linear_contain_FT}}
\label{fig:dist}
\end{figure}

Clearly, by taking the rows of $W$ to point to the middle of each sector (dashed arrows in the illustration), we  get $\Err^*_\cd(\cl)=0$. In contrast, no linear separator can split the three labels into two gropus without error, thus $\Err^*_\cd(\cw_\trees)>0$.

Finally, to see that $\cl$ essentially contains $\cw_{\OvA}$, we note that $\cw_{\OvA}=\cw_T$ where $T$ is a tree such that each of its internal nodes has a leaf corresponding to one of the labels as its left son. Thus $\cw_\OvA$ is essentially contained in $\cw_\trees$.
\end{proof}

\begin{proof}[Proof of Theorem \ref{thm:AP_contain_Linear}]
It is easily seen that $\cw_{AP}$ contains $\cl$: Let $W \in \reals^{d+1\times k}$, and denote its i'th row by $W[i]$. For each column $(i,j)$ of $M^{\AP}$, define the binary classifier $h_{i,j} \in \cw$ such that $\forall x \in \reals^d, h_{i,j}(\bar{x}) = \sign(\dotprod{W[j] - W[i], \bar{x}})$.
Then for all $x$, $h_W(x) = \tilde{M^{\AP}}(h_{1,1}(x),\ldots,h_{k-1,k}(x))$.

To show that the inclusion is strict, as in the proof of Theorem \ref{thm:Linear_contain_FT}, we can and will assume that $d=2$. Choose $k^*$ to be the minimal number such that for every $k\ge k^*$, $d_N(\cw_{AP})>d_N(\cl)$: This number exists by Theorems \ref{thm:sample_OvA} and \ref{thm:sample_linear} (note that though we chose $k^*$ w.r.t. $d=2$, the same $k^*$ is valid for every $d$). For any $k\ge k^*$, it follows that there is a set $S\subseteq \reals^2$ that is $N$-shattered by $\cw_{AP}$ but not by $\cl$. Thus, there is a hypothesis $h\in \cw_{AP}$ such that for every $g\in \cl$, $g|_S\ne h|_S$. Define the distribution $\cd$ to be uniform over $\{(x,h(x)):x\in S\}$. Then clearly $\Err^*_\cd(\cl)>\Err^*_\cd(\cw_{AP}) = 0$.
\end{proof}

Next, we prove Theorem \ref{thm:embedding}, which we restate more formally as follows. Note that the result on OvA is implied since there exists a tree that implements OvA.
\begin{theorem}(Restatement of Theorem \ref{thm:embedding})
If there exists an embedding $\iota:\reals^d \rightarrow \reals^{d'}$ and a tree $T$ such that $\cw^{d'}_T\circ \iota$
essentially contains $\cl$, then necessarily $d' \geq \tilde \Omega(dk)$.
\end{theorem}
\begin{proof}
Assume that $i\in [k]$ is the class corresponding to the leaf with the least depth, $l$. Note that $l\le \log_2(k)$. Let $\phi:[k]\to\{\pm 1\}$ be the function that is $1$ on $\{i\}$ and $-1$ otherwise. It is not hard to see that $\phi \circ \cl$ is the hypothesis class of convex polyhedra in $\reals^d$ having $k-1$ faces. Thus,
\begin{equation}\label{eq:4}
\VC(\phi \circ \cl)\ge (k-1)d,
\end{equation}
\citep[see e.g.][]{Tak09}. On the other hand, $\phi\circ\cw^{d'}_T$, is the class of convex polyhedra in $\reals^{d'}$ having $l\le\log_2(k)$ faces. Thus, by Lemma \ref{lemma:meta_lemma}
\begin{equation}\label{eq:5}
 \VC(\phi \circ \cw^{d'}_T\circ \iota) \le \VC(\phi \circ \cw^{d'}_T)\le O(ld'\log(ld'))\le O(\log(k)d'\log(\log(k)d'))
\end{equation}
By the assumption that $\cw^{d'}_T\circ \iota$
essentially contains $\cl$, $\VC(\phi\circ\cl)\le\VC(\phi\circ\cw^{d'}_T\circ \iota)$. Combining with equations (\ref{eq:4}) and (\ref{eq:5}) it follows that $d(k-1)= O(\log(k)d'\log(\log(k)d'))$. Thus, $d'=\tilde\Omega\left(d k\right)$.
\end{proof}

To prove Lemma \ref{lemma:split_classes}, we first state the classic VC-dimension theorem, which will be useful to us.
\begin{theorem} [\cite{Vapnik98}]\label{thm:UC_VC}
There exists a constant $C>0$ such that for every hypothesis class $\ch\subseteq \{\pm 1\}^\cx$ of VC dimension $d$, a distribution $\cd$ over $\cx$, $\epsilon,\delta>0$ and $m\ge C\frac{d+\ln(\frac{1}{\delta})}{\epsilon^2}$ we have
\[
\Pr_{S\sim \cd^m}\left[\operatorname{Err}^*_\cd(\ch)\ge \inf_{h\in \ch}\operatorname{Err}_{S}(h)-\epsilon\right]\ge 1-\delta.
\]
\end{theorem}

We also use the following lemma, which proves a variant of Hoeffding's inequality.
\begin{lemma}\label{lem:chern_substitute}
Let $\beta_1,\ldots,\beta_k\ge 0$ and let $\gamma_1,\ldots,\gamma_k\in\reals$, such that $\forall i,\;|\gamma_i|\le\beta_i$. Fix an integer $j \in \{1,\ldots,\lfloor\frac{k}{2}\rfloor\}$ and let $\mu=j/k$.
Let $(X_1,\ldots,X_k)\in \{\pm 1\}^k$ be a random vector sampled uniformly from the set $\{(x_1,\ldots,x_k):\sum_{i=1}^k\frac{x_1+1}{2}=\mu k\}$. Define $Y_i=\beta_i+X_i\gamma_i$ and denote $\alpha_i=\beta_i+|\gamma_i|$. Assume that $\sum_{i=1}^k\alpha_i=1$. Then
\[
\Pr\left[\sum_{i=1}^kY_i\le \mu-\epsilon \right]\le 2\exp\left(-\frac{\epsilon^2}{2\sum_{i=1}^k\alpha_i^2}\right).
\]
\end{lemma}
\begin{proof}
First, since $\mu<\frac{1}{2}$, it suffices to prove the claim for the case $\forall i,\gamma_i\ge 0$ since this is the ``harder'' case. Let $Z_1,\ldots,Z_k\in\{\pm 1\}$ be independent random variables such that $\Pr[Z_i=1]=\mu-\frac{\epsilon}{2}$. Denote $W_i=\beta_i+Z_i\gamma_i$. Further denote $\bar W=\sum_{i=1}^kW_i$ and $\bar Z=\sum_{i=1}^k\frac{Z_i+1}{2}$.

Note that for every $j_0\le j=\mu k$, given that $\bar Z=j_0$, $\bar W$ can be described as follows: We start with the value $\sum_{i=1}^k\beta_i-\gamma_i$ and then choose $j_0$ indices uniformly from $[k]$. For each chosen index $i$, the value of $\bar W$ is increased by $2\gamma_i$. $\sum_{i=1}^kY_i$ can be described in the same way, except that that $j\ge j_0$ indices are chosen. Thus, $\Pr\left[\sum_{i=1}^kY_i\le \mu -\epsilon\right]\le \Pr\left[\bar W\le \mu -\epsilon \mid \bar Z= j_0\right]$. Thus, we have
\begin{align*}
\Pr\left[\sum_{i=1}^kY_i\le \mu -\epsilon\right] &\le
\Pr\left[\bar W\le \mu -\epsilon \mid \bar Z\le \mu k\right]
\\
&\le \Pr\left[\bar W\le \mu -\epsilon\right]/\Pr\left[\bar Z\le \mu k \right]
\\
&\le 2\Pr\left[\bar W\le \mu -\epsilon\right]
\\
&\le 2\exp\left(-\frac{\epsilon^2}{2\sum_{i=1}^k\alpha_i^2}\right).
\end{align*}
The last inequality follows from Hoeffding's inequality and noting that
\[
E[W_i]=\beta_i+(2(\mu-\frac{\epsilon}{2})-1)\gamma_i
=(\mu-\frac{\epsilon}{2})(\beta_i+\gamma_i)+(1-\mu+\frac{\epsilon}{2})(\beta_i-\gamma_i)
\ge (\mu-\frac{\epsilon}{2})\alpha_i.
\]
So that $\sum_{i=1}^kE[W_i]\ge (\mu-\frac{\epsilon}{2})\sum_{i=1}^k\alpha_i=\mu-\frac{\epsilon}{2}$.
\end{proof}

\begin{proof}[Proof of Lemma \ref{lemma:split_classes}]

The idea of this proof is as follows:
Using a uniform convergence argument based on the VC dimension of the binary hypothesis class, we show that there exists a labeled sample $S$
such that $|S|\approx \frac{d+k}{\nu^2}$, and for all possible mappings
$\phi$, the approximation error of the hypothesis class on the sample is close to the approximation error on the distribution $\cd_\phi$. This allows us to restrict our attention to
a finite set of hypotheses, based on their restriction to the sample. For these hypotheses,
we show that with high probability over the choice of $\phi$, the approximation error on the sample is high.
Using a union bound on the possible hypotheses, we conclude that the approximation error on the distribution
will be high, with high probability over the choice of $\phi$.

For $i\in [k]$, denote $p_i=\Pr_{x\sim\cd}[f(x)=i]$. Let $S=\{(x_1,y_1),\ldots,(x_m,y_m)\}\subseteq\cx\times [k]$ be an
i.i.d. sample drawn according to $\cd$ where $m=\lceil C\frac{d+(k+2)\ln(2)}{(\nu/2)^2}\rceil$, for the constant from $C$ from Theorem \ref{thm:UC_VC}.
Given $S$, denote $S_\phi\{(x_1,\phi(y_1)),\ldots,(x_m,\phi(y_m))\}\subseteq\cx\times\{\pm 1\}$. For $i\in [k]$, let $\hat p_i=\frac{|\{j:y_j=i\}|}{m}$.

For any fixed $\phi:[k]\to\{\pm1\}$, with probability $>1-2^{-(k+2)}$ over the choice of $S$ we have, by Theorem \ref{thm:UC_VC}, that
$\operatorname{Err}^*_{\cd_\phi}(\ch)>\inf_{h\in\ch}\operatorname{Err}_{S_\phi}( h)-\nu$. Since $|\{\pm 1\}^{[k]}|=2^k$, w.p. $>1-\frac{1}{4}$,

\begin{equation}\label{eq:1}
\forall \phi\in\{\pm 1\}^{[k]},\;\; \operatorname{Err}^*_{\cd_\phi}(\ch)>\inf_{h\in\ch}\operatorname{Err}_{S_\phi}( h)-\frac{\nu}{2}.
\end{equation}
Moreover, we have
\[
E[\sum_{i=1}^k\hat{p}_i^2]=\frac{1}{m^2}\sum_{i=1}^k\left(\binom m2p_i^2+mp_i\right)\le k\cdot\left(\frac{m(m-1)}{2m^2}\frac{100}{k^2}+\frac{10}{mk}\right)\le \frac{60}{k}.
\]
Thus, by Markov's inequality, w.p. $\ge\frac{1}{2}$ we have
\begin{equation}\label{eq:2}
\sum_{i=1}^k\hat{p}_i^2<\frac{120}{k}.
\end{equation}
Thus, with probability at least $1-\frac{1}{4}-\frac{1}{2}>0$, both (\ref{eq:2})  and (\ref{eq:1}) holds. In particular, there exists a sample $S$ for which both (\ref{eq:2})  and (\ref{eq:1}) hold. Let us fix such an $S=\{(x_1,y_1),\ldots,(x_m,y_m)\}$.

Assume now that $\phi\in \{\pm 1\}^{[k]}$ is sampled according to the first condition. Denote
\[
Y_i=|\{j:h(x_j)\ne \phi(y_j)\text{ and } y_j=i\}|/m.
\]
For a fixed $h\in\ch$ we have
\[
\Pr_{\phi}\left[\operatorname{Err}_{S_\phi}( h)<\mu-\frac{\nu}{2}\right]=\Pr_{\phi}\left[\sum_{i=1}^kY_i<\mu-\frac{\nu}{2}\right]
\]
We note that $Y_i$ are independent random variables with $E[Y_i]\ge \mu\hat{p}_i$ and $0\le Y_i\le \hat{p}_i$. Thus, by Hoeffding's inequality,
\[
\Pr_{\phi}\left[\operatorname{Err}_{S_\phi}( h)<\mu-\frac{\nu}{2}\right]\le \exp\left(-\frac{\nu^2}{2\sum_{i=1}^k\hat{p}_i^2}\right)\le \exp\left(-\frac{\nu^2k}{240}\right).
\]
By Sauer's lemma, $|\ch|_{\{x_1,\ldots,x_m\}}|\le \left(\frac{em}{d}\right)^d$. Thus, with probability $\ge 1-\left(\frac{em}{d}\right)^d\exp\left(-\frac{\nu^2k}{240}\right)$ over the choice of $\phi$, $\inf_{h\in\ch}\operatorname{Err}_{S_\phi}( h)\ge \mu-\frac{\nu}{2}$ and by (\ref{eq:1}) also
\begin{equation}\label{eq:3}
\operatorname{Err}^*_{\cd_\phi}( \ch)\ge \frac{1}{2}-\nu.
\end{equation}
Finally, since $m=O\left(\frac{k+d}{\nu^2}\right)$, if $k=\Omega\left(\frac{d\ln(1/\nu)+\ln(1/\delta)}{\nu^2}\right)$ then Eq. (\ref{eq:3}) holds w.p $>1-\delta$, concluding the proof for the case when the first condition holds.
If the second condition holds, the proof is very similar, with the sole difference that Lemma \ref{lem:chern_substitute} is used instead of Hoeffding's inequality.
\end{proof}

\begin{proof}[Proof of Corollary \ref{cor:FT_approx_poorly}]
The Corollary follows from Lemma \ref{lemma:split_classes}, by noting that $\Err^*_\cd(\ch_T)\ge\Err^*_{\cd_\phi}(\ch)$, where $\phi:[k]\to\{\pm 1\}$ is defined as $\phi(i)=1$ if and only if $\lambda^{-1}(i)$ is in the right subtree emanating from the root of $T$.
\end{proof}

\begin{proof}[Proof of Corollary \ref{cor:ECOC_approx_poorly}]
Let $\phi:[k]\to\{\pm 1\}$ be the function that is $-1$ on $\left[\lfloor\frac{k}{2}\rfloor\right]$ and $1$ otherwise. By Lemma \ref{lemma:meta_lemma}, applied to $L(\ch) = \phi\circ\ch_{(M,\mathrm{Id})}$, $\VC(\phi\circ\ch_{(M,\mathrm{Id})})=O(dl\log(dl))$, so that, by Lemma \ref{lemma:split_classes} (applied to a random choice of $\lambda$ instead of $\phi$),
$\Err_{\cd_{\phi \circ \lambda}}^*(\phi\circ\ch_{(M,\mathrm{Id})})\ge \frac{1}{2}-\nu$ with probability $>1-\delta$ over the choice of $\lambda$. The proof follows as we note that for every $\lambda$, $\Err_{\cd}^*(\ch_{(M,\lambda^{-1})}) = \Err_{\cd_\lambda}^*(\ch_{(M,Id)}) \ge\Err_{\cd_{\phi \circ \lambda}}^*(\phi\circ\ch_{(M,\mathrm{Id})})$.
\end{proof}



\end{document}